\documentclass[]{article}

\usepackage{times}
\usepackage{amsmath}
\usepackage{amssymb}
\usepackage{amsfonts}
\usepackage{mathrsfs}
\usepackage{subfigure}
\usepackage{amsthm}
\usepackage{fixme}
\usepackage{url}
\usepackage{sidecap}
\usepackage{authblk}
\usepackage[top=1.9cm, bottom=1.9cm, left=2cm, right=2cm]{geometry}
\usepackage{cite}
\usepackage{graphicx}
\usepackage{color}
\usepackage{pstricks}

\newrgbcolor{aasacolor}{.7 .1 .1}

\newcommand{\R}{\mathbb{R}}

\newcommand{\N}{\mathbb{N}}

\newenvironment{packed_itemize}{
\begin{itemize}
  \setlength{\itemsep}{1pt}
  \setlength{\parskip}{0pt}
  \setlength{\parsep}{0pt}
}{\end{itemize}}

\newcounter{example}
\newcounter{theorem}

\newcounter{corollary}
\newcounter{lemma}
\newcounter{definition}

\newtheorem{thm}[theorem]{Theorem}
\newtheorem{lem}[lemma]{Lemma}

\newtheorem{cor}[corollary]{Corollary}
\newtheorem{exA}[example]{Example}
\newtheorem{defnA}[definition]{Definition}

\newenvironment{defn}{\begin{defnA}\rm}{\end{defnA}}

\newtheorem{remA}[equation]{Remark}
\newenvironment{rem}{\begin{remA}\rm}{\end{remA}}
\newenvironment{ex}{\begin{exA}\rm}{\end{exA}}

\newcommand{\etal}[0]{et al.\ }
\newcommand{\eg}[0]{e.g.\ }

\usepackage[pagebackref=false,breaklinks=true,letterpaper=true,colorlinks,bookmarks=false]{hyperref}

\begin{document}

\title{Geodesic Exponential Kernels: When Curvature and Linearity Conflict}

\author[1]{Aasa Feragen \thanks{\texttt{aasa@diku.dk}}}
\affil[1]{Department of Computer Science, University of Copenhagen}

\author[1]{Fran\c{c}ois Lauze \thanks{\texttt{francois@diku.dk}}}

\author[2]{S\o{}ren Hauberg \thanks{\texttt{sohau@dtu.dk}}}
\affil[2]{Section for Cognitive Systems, DTU Compute}

\maketitle

\begin{abstract}
We consider kernel methods on general geodesic metric spaces and provide both negative and positive results. First we show that the common Gaussian kernel can only be generalized to a positive definite kernel on a geodesic metric space if the space is flat. As a result, for data on a Riemannian manifold, the geodesic Gaussian kernel is only positive definite if the Riemannian manifold is Euclidean. This implies that any attempt to design geodesic Gaussian kernels on curved Riemannian manifolds is futile. However, we show that for spaces with conditionally negative definite distances the geodesic Laplacian kernel can be generalized while retaining positive definiteness. This implies that geodesic Laplacian kernels can be generalized to some curved spaces, including spheres and hyperbolic spaces. Our theoretical results are verified empirically.
\end{abstract}


\section{Introduction}

Standard statistics and machine learning tools require input data residing in a
Euclidean space. However, many types of data are more faithfully represented in general
nonlinear metric spaces (\eg Riemannian manifolds). This is, for instance, the case
for shapes~\cite{Kendall:BLMS:1984, Taheri:AFGR:2011, Freifeld:ECCV:2012, Srivastava:PAMI:2005, Younes:ICV:2012,hinkle,bronsteins},
DTI images \cite{lenglet:eccv:2004, Pennec:IJCV:2006,fletcher,wang2014tracking}, motion models~\cite{turaga,Cetingul:cvpr:2009},
symmetric positive definite matrices \cite{Tuzel:ECCV:2006, Porikli:CVPR:2006,carreira}, illumination-invariance~\cite{caseiro},
human poses \cite{Murray:Book:1994, Hauberg:IMAVIS:2011},
tree structured data \cite{feragen:pami:2013, feragen:ipmi:2013},
metrics \cite{hauberg:nips:2012, gil1992riemannian}
and probability distributions~\cite{amari:book:2000}. The underlying metric space captures domain specific knowledge, \eg non-linear constraints, which is available \emph{a priori}. The intrinsic geodesic metric encodes this knowledge, often leading to improved statistical models. 

A seemingly straightforward approach to statistics in metric spaces is to use kernel methods~\cite{learningwithkernels}, designing kernels $k(x, y)$ which only rely on geodesic distances $d(x, y)$ between observations~\cite{chapelle1999support}:
\begin{align} \label{eq:geod_kerneldef}
  k(x, y) &= \exp\left( -\lambda (d(x, y))^q \right), \qquad \lambda, q > 0.
\end{align}
For $q = 2$ this gives a geodesic generalization of the \emph{Gaussian kernel}, and $q = 1$ gives the geodesic \emph{Laplacian kernel}. While this idea has an appealing similarity to familiar Euclidean kernel methods, we show that it is highly limited if the metric space is curved. 

Positive definiteness of a kernel $k$ is critical for the use of kernel methods such as support vector machines or kernel PCA. In this paper, we analyze exponential kernels on geodesic metric spaces and show the following results, summarized in Table~\ref{tab:overview}.
\begin{packed_itemize}
  \item The \emph{geodesic Gaussian kernel} is positive definite (PD) for all $\lambda > 0$ only if the underlying metric space is flat (Theorem~\ref{Xisflat}). In particular, when the metric space is a Riemannian manifold, the geodesic Gaussian kernel is PD for all $\lambda > 0$ if and only if the manifold is Euclidean (Theorem~\ref{isomeucl}). This negative result implies that Gaussian kernels cannot be generalized to any non-trivial Riemannian manifolds of interest.
  \item The \emph{geodesic Laplacian kernel} is PD if and only if the metric is conditionally negative definite (Theorem~\ref{thm:PDlaplacian}). This condition is not generally true for metric spaces, but it holds for a number of spaces of interest. In particular, the geodesic Laplacian kernel is PD on spheres, hyperbolic spaces, and Euclidean spaces (Table~\ref{tab:PD_or_not}).
  \item For any Riemannian manifold $(M, g)$, the kernel~\eqref{eq:geod_kerneldef} will never be PD for all $\lambda > 0$ if $q > 2$ (Theorem~\ref{nothigherthan2}). 
\end{packed_itemize}
\begin{table}
\centering
\resizebox{0.5\columnwidth}{!}{
  \begin{tabular}{l|c|c}
                                                  & \multicolumn{2}{c}{\textbf{Extends to general}}           \\
    \footnotesize{\textbf{Kernel}}                & \textbf{Metric spaces} &   \textbf{Riemannian manifolds}  \\
    \hline
    \footnotesize{Gaussian ($q = 2$)}             & No (only if flat)    &   No (only if Euclidean)       \\
    \footnotesize{Laplacian ($q = 1$)}            & Yes, iff metric is CND &   Yes, iff metric is CND         \\
    \footnotesize{Geodesic exp. ($q > 2$)}        & Not known              &   No
  \end{tabular}
  }
  \caption{Overview of results: For a geodesic metric, when is the geodesic exponential kernel~\eqref{eq:geod_kerneldef} positive definite for all $\lambda > 0$?}
  \label{tab:overview}
\end{table}

Generalization of geodesic kernels to metric spaces is motivated by the
general lack of powerful machine learning techniques in these spaces. In that
regard, our first results are disappointing as they imply that
generalizing Gaussian kernels to metric spaces is \emph{not} a viable direction
forward. Intuitively, this is not surprising as kernel methods embed
the data in a linear space, which cannot be expected to capture the curvature
of a general metric space. Our second result is therefore a positive surprise:
it allows the Laplacian kernel to be applied in \emph{some} metric spaces, although this has strong implications for their geometric properties.
This gives hope that other kernels can be generalized, though our third
result indicates that the geodesic exponential kernels \eqref{eq:geod_kerneldef} have limited applicability on Riemannian manifolds. 

{\bf The paper is organized as follows.} We state our main results and discuss their consequences in Sec.~\ref{sec:mainres}, postponing proofs until Sec.~\ref{sec:mathsec}, which includes a formal discussion of the preliminaries. This section can be skipped in a first reading of the paper. Related work is discussed in detail in Sec.~\ref{sec:relatedwork}, where we also review recent approaches which do not conflict with our results. Sec.~\ref{sec:experiments} contains empirical experiments confirming and extending our results on manifolds that admit PD geodesic exponential kernels. The paper is concluded in Sec.~\ref{sec:discussion}.

\section{Main results and their consequences} \label{sec:mainres}
Before formally proving our main theorems, we state
the results and provide hints as to \emph{why} they hold. We start with a brief
review of \emph{metric geometry} and the notion of a \emph{flat space}, both
of which are fundamental to the results.



\begin{figure}
  \centering
  \includegraphics[width=0.4\columnwidth]{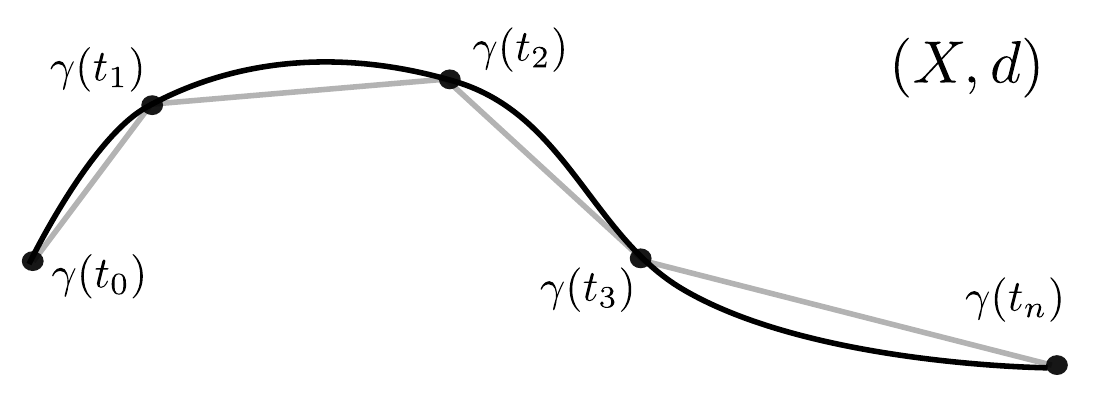}
  \vspace{-2mm}
  \caption{Path length in a metric space is defined as the supremum
    of lengths of finite approximations of the path.}
  \label{fig:path_length}
\end{figure}

In a general metric space $(X, d)$ with distance metric $d$, the \emph{length} $l(\gamma)$ of a path $\gamma \colon [0, L] \to X$ from $x$ to $y$ is defined as the smallest upper bound of any finite
approximation of the path (see Fig.~\ref{fig:path_length})
\[
l(\gamma) = \sup_{0 = t_0 < t_1 < \ldots < t_n = 1, n \in \N} \sum_{i = 1}^n d(t_{i-1}, t_i).
\]
A path $\gamma \colon [0, L] \to X$ is called a \emph{geodesic}~\cite{bridsonhaef} from $x$ to $y$ if $\gamma(0) = x$, $\gamma(L) = y$ and $d\left(\gamma(t), \gamma(t') \right) = |t-t'|$ for all $t, t' \in [0, L]$. In particular, $l(\gamma) = d(x,y) = L$ for a geodesic $\gamma$. In a Euclidean space, geodesics are straight lines. A geodesic from $x$ to $y$ will always be the shortest possible path from $x$ to $y$, but geodesics with respect to a given metric do not always exist, even if shortest paths do. An example is given later in Fig.~\ref{fig:chordal_fig}.

A metric space $(X, d)$ is called a \emph{geodesic space} if every pair $x, y \in X$ can be connected by a geodesic. Informally, a geodesic metric space is merely a space in which distances can be computed as lengths of geodesics, and data points can be interpolated via geodesics.


\emph{Riemannian manifolds} are a commonly used class of metric spaces. Here distances
are defined locally through a smoothly changing inner product in the tangent space. Intuitively, a Riemannian manifold can be thought of as a smooth surface (\eg a sphere)
with geodesics corresponding to shortest paths on the surface. A geodesic distance metric corresponding to the Riemannian structure is defined explicitly as the length of the geodesic joining two points. Whenever a Riemannian manifold is complete, it is a geodesic space.
This is the case for most manifolds of interest.

Many efficient machine learning algorithms are available in Euclidean spaces; their generalization to metric spaces is an open problem.
\emph{Kernel methods} form an immensely popular class of algorithms
including \emph{support vector machines} and \emph{kernel PCA}~\cite{learningwithkernels}.
These algorithms rely on the specification of a kernel $k(x, y)$, which embeds data
points $x, y$ in a linear Hilbert space and returns their inner product.
Kernel methods are very flexible, as they only require the computation of inner products (through the kernel). However, the kernel is only an inner product if it is PD, so kernel methods are only well-defined for kernels which are PD~\cite{learningwithkernels}.

Many popular choices of kernels for Euclidean data rely only on the Euclidean
distance between data points; for instance the widely used Gaussian kernel (given by \eqref{eq:geod_kerneldef} with $q=2$). Kernels which only rely on distances form an obvious target for generalizing kernel methods to metric spaces,
where distance is often the only quantity available.


\subsection{Main results}

In Theorem~\ref{Xisflat} of this paper we prove that {\bf geodesic Gaussian kernels
on metric spaces are PD for all $\lambda > 0$ only if the metric space is flat}.
Informally, a metric space is flat if it (for all practical purposes) is Euclidean.
More formally:

\begin{defn}
A geodesic metric space $(X, d)$ is \emph{flat in the sense of Alexandrov} if 
any geodesic triangle in $X$ can be isometrically embedded in a Euclidean space.
\end{defn}

Here, an embedding $f \colon X \to X'$ from a metric space $(X, d)$ to another metric space $(X', d')$ is \emph{isometric} if $d'\left(f(x), f(y) \right) = d(x,y)$ for all $x, y \in X$\footnote{The metric space definition of \emph{isometric embedding}~\cite{bridsonhaef}, which is used when distances are in focus, should not be confused with the definition of isometric embedding from Riemannian geometry, preserving \emph{Riemannian metrics} which are not distances, but tangent space inner products.}. A \emph{geodesic triangle} $abc$ in $X$ consists of three points $a$, $b$ and $c$ joined by geodesic paths $\gamma_{ab}$, $\gamma_{bc}$ and $\gamma_{ac}$. The concept of flatness essentially requires that all geodesic triangles
are \emph{identical} to Euclidean triangles; see Fig.~\ref{geodesic_fig}.

\begin{figure}
  \centering
  \includegraphics[width=0.5\columnwidth]{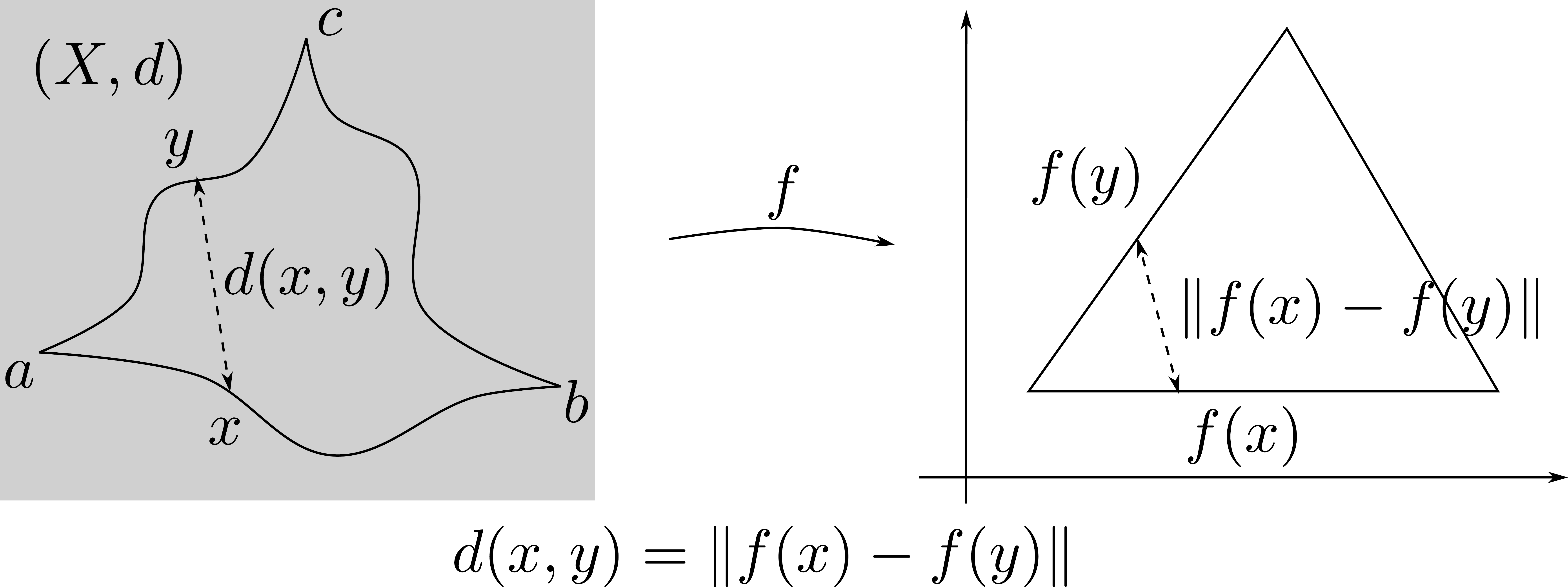}
  \caption{If any geodesic triangle in $(X,d)$ can be isometrically embedded
    into some Euclidean space, then $X$ is \emph{flat}. Note in particular that
    when a geodesic triangle is isometrically embedded in a Euclidean space, it
    is embedded onto a Euclidean triangle --- otherwise the geodesic edges would
    not be isometrically embedded.}
  \label{geodesic_fig}
\end{figure}

With this, we state our first main theorem:

\begin{thm} \label{Xisflat}
Let $(X, d)$ be a geodesic metric space, and assume that $k(x,y) = \exp(-\lambda d^2(x,y))$
is a PD geodesic Gaussian kernel on $X$ for all $\lambda > 0$. Then $(X, d)$ is flat in the sense of Alexandrov.
\end{thm}

This is a negative result, in the sense that most metric spaces of interest are
not flat. In fact, the motivation for generalizing kernel methods is to
cope with data residing in non-flat metric spaces.

As a consequence of Theorem~\ref{Xisflat}, we show that
{\bf geodesic Gaussian kernels on Riemannian manifolds are PD for all $\lambda > 0$ only if the Riemannian manifold is Euclidean}.

\begin{thm} \label{isomeucl}
Let $M$ be a complete, smooth Riemannian manifold with its associated geodesic distance metric $d$. Assume, moreover, that $k(x,y) = \exp(-\lambda d^2(x,y))$ is a PD geodesic Gaussian kernel for all $\lambda > 0$. Then the Riemannian manifold $M$ is isometric to a Euclidean space.
\end{thm}

{\bf These two theorems have several consequences.} The first and main consequence is that defining geodesic Gaussian kernels on Riemannian manifolds or other geodesic metric spaces has limited applicability as most spaces of interest are not flat. In particular, on Riemannian manifolds the kernels will generally only be PD if the original data space is Euclidean. In this case, nothing is gained by treating the data space as a Riemannian manifold, as it is perfectly described by the well-known Euclidean geometry, where many problems can be solved in closed form. In Sec.~\ref{sec:relatedwork} we re-interpret recent work which does, indeed, take place in Riemannian manifolds that turn out to be Euclidean.


%

Second, \emph{this result is not surprising:} Curvature cannot be captured by a flat space, and Sch\"{o}nberg's classical theorem (see Sec.~\ref{sec:kernels}) indicates a strong connection between PD Gaussian kernels and linearity of the distance measure used in the Gaussian kernel. This connection is made explicit by Theorems~\ref{Xisflat} and~\ref{isomeucl}.

The obvious next question is the extent to which these negative results depend on the choice $q=2$ in~\eqref{eq:geod_kerneldef}, which results in a Gaussian kernel. A recent result by Istas~\cite{istas} implies that for Riemannian manifolds, passing to a higher power $q > 2$ will \emph{never} lead to a PD kernel for all $\lambda > 0$:

\begin{thm} \label{nothigherthan2}
Let $M$ be a Riemannian manifold with its associated geodesic distance metric $d$, and let $q > 2$. Then there is some $\lambda > 0$ so that the kernel~\eqref{eq:geod_kerneldef} is not PD.
\end{thm}

The existence of a $\lambda > 0$ such that the kernel is not PD may seem innocent; however, this implies that the kernel bandwidth parameter cannot be learned.

In contrast, the choice $q=1$ in~\eqref{eq:geod_kerneldef}, giving a geodesic Laplacian kernel, leads to a more positive result: The geodesic Laplacian kernel will be positive definite if and only if the distance $d$ is \emph{conditionally negative definite} (CND). CND metrics have linear embeddability properties analogous to those of PD kernels; see Sec.~\ref{sec:kernels} for formal definitions and properties. This provides a PD kernel framework which can, for several popular Riemannian data manifolds, take advantage of the geodesic distance.

\begin{thm} \label{thm:PDlaplacian}
\begin{itemize}
\item[i)] The geodesic distance $d$ in a geodesic metric space $(X, d)$ is CND if and only if the corresponding geodesic Laplacian kernel is PD for all $\lambda > 0$. 
\item[ii)] In this case, the square root metric $d_{\sqrt{\ }}(x,y) = \sqrt{d(x,y)}$ is also a distance metric, and $(X, d_{\sqrt{\ }})$ can be isometrically embedded as a metric space into a Hilbert space $H$. 
\item[iii)] The square root metric $d_{\sqrt{\ }}$ is not a geodesic metric, and $d_{\sqrt{\ }}$ corresponds to the chordal metric in $H$, not the intrinsic metric on the image of $X$ in $H$.
\end{itemize}
\end{thm}

In Theorem~\ref{thm:PDlaplacian}, for $\phi \colon X \to H$, the \emph{chordal metric} $\|\phi(x) - \phi(y)\|_H$ measures distances directly in $H$ rather than intrinsically in the image $\phi(X) \subset H$, see also Fig.~\ref{fig:chordal_fig}.

\begin{SCfigure}[2][t]
\includegraphics[width=0.2\linewidth]{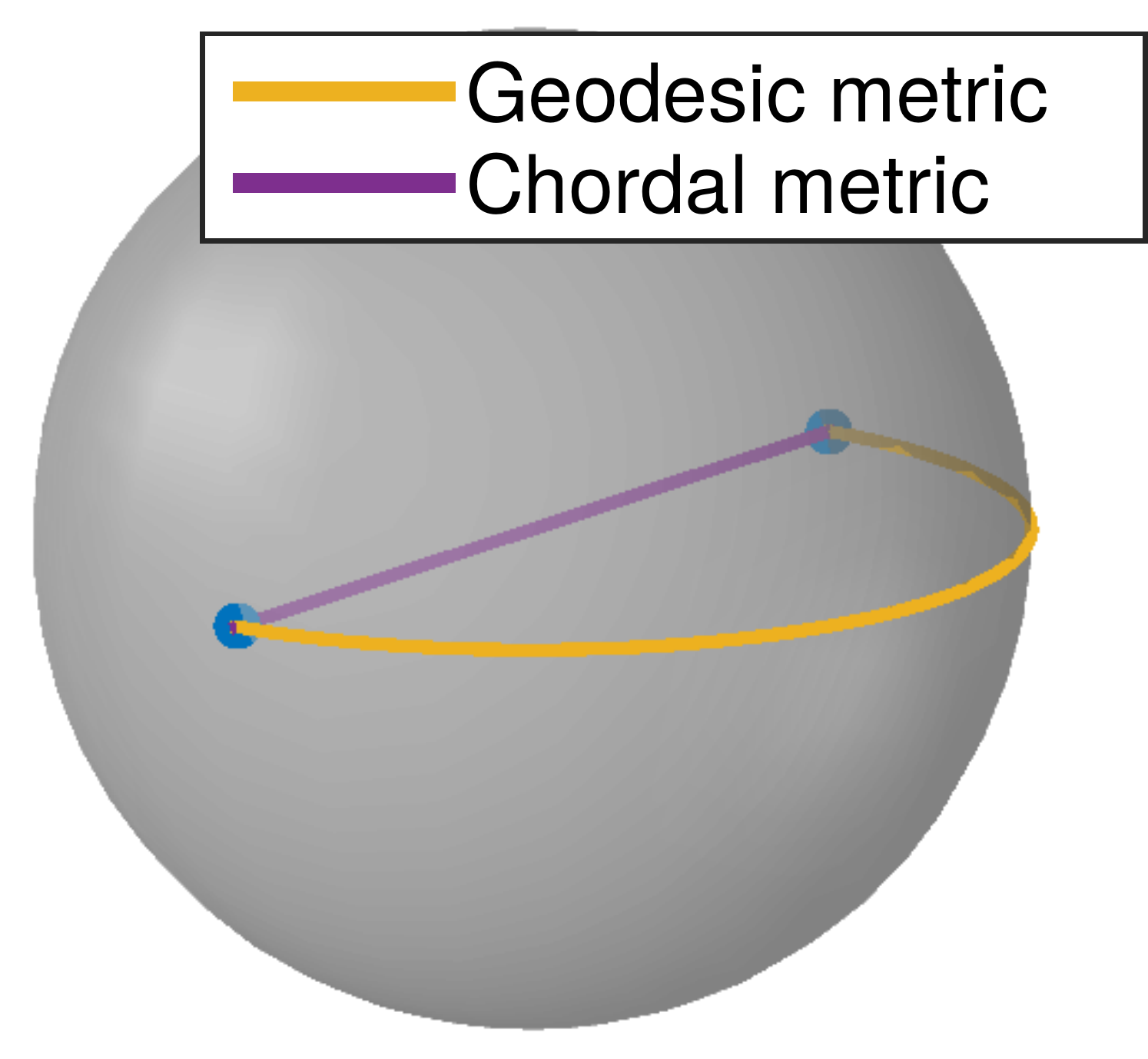}
\caption{The chordal metric on $\mathbb{S}^2 \subset \R^3$ is measured directly in $\R^3$, while the geodesic metric is measured along $\mathbb{S}^2$. Shortest paths with respect to the two metrics coincide, but the chordal metric is \emph{not} a geodesic metric, and the shortest path is \emph{not} a geodesic for the chordal metric, because the shortest path between two points is longer than their chordal distance.}
\label{fig:chordal_fig}
\end{SCfigure}

In Sec.~\ref{sec:relatedwork} we discuss several popular data spaces for which geodesic Laplacian kernels are PD (see Table~\ref{tab:PD_or_not}); examples include spheres, hyperbolic spaces and more. Nevertheless, we see from part \emph{ii)} of Theorem~\ref{thm:PDlaplacian} that any geodesic metric space whose geodesic Laplacian kernel is always PD must necessarily have strong linear properties: Its square root metric is isometrically embeddable in a Hilbert space.

This illustrates an intuitively simple point: A PD kernel has no choice but to linearize the data space. Therefore, its ability to capture the original data space geometry is deeply connected to the linear properties of the original metric\footnote{Another curious connection between kernels and curvature is found in~\cite{burges1999}, which shows that Gaussian and polynomial kernels on $\R^n$ and $\R^2$, respectively, have flat feature space images $\phi(\R^n)$ and $\phi(\R^2)$.}.
%

\section{Proofs of main results} \label{sec:mathsec}

In this section we prove the main results of the paper; this section may be skipped
in a first reading of the paper. In the first two subsections
we review and discuss classical geometric results on kernels, manifolds and curvature, which we will use to prove the main results.

\subsection{Kernels}\label{sec:kernels}

A modern and comprehensive treatment of the classical results on PD and CND kernels referred to here, can be found in~\cite[Appendix C]{kazhdan}.

\begin{defn}
A positive definite (PD) kernel on a topological space $X$ is a continuous function
\[
k \colon X \times X \to \R
\]
such that for any $n \in N$, any elements $x_1, \ldots, x_n \in X$ and any numbers $c_1, \ldots, c_n \in \R$, we have
\[
\sum_{i=1}^n \sum_{j=1}^n c_i c_j k(x_i, x_j) \ge 0.
\]
\end{defn}

\begin{defn}
A conditionally negative definite (CND) kernel on a topological space $X$ is a continuous function
\[
\psi \colon X \times X \to \R
\]
which satisfies
\begin{packed_itemize}
\item[i)] $\psi(x,x) = 0$ for all $x \in X$
\item[ii)] $\psi(x,y) = \psi(y,x)$ for all $x, y \in X$
\item[iii)] for any $n \in \N$, any elements $x_1, \ldots, x_n \in X$ and any real numbers $c_1, \ldots, c_n$ with $\sum_{i=1}^n c_i = 0$, we have
\[
\sum_{i=1}^n \sum_{j=1}^n c_i c_j \psi(x_i, x_j) \le 0.
\]
\end{packed_itemize}
\end{defn}

\begin{ex}
If $d \colon H \times H \to \R$ is the metric induced by the norm on a Hilbert space $H$, then the map $d^2 \colon H \times H \to \R$ given by $d^2(x,y) = (d(x,y))^2$ is a CND kernel~\cite{kazhdan}.
\end{ex}

The following two theorems are key to understanding the connection between distance metrics and their corresponding exponential kernels.

\begin{thm}[Due to Sch\"onberg \cite{schonberg}, Theorem C.3.2 in \cite{kazhdan}]
If $X$ is a topological space and $\psi \colon X \times X \to \R$ is a continuous kernel on $X$ such that $\psi(x,x) = 0$ and $\psi(y,x) = \psi(x,y)$ for all $y, x \in X$, then the following two properties are equivalent:
\begin{packed_itemize}
\item $\psi$ is a CND kernel  \label{thm:schonberg}
\item the kernel $k(x,y) = e^{-\lambda \psi(x,y)}$ is PD for all $\lambda \ge 0$. \qed
\end{packed_itemize}
\end{thm}

\begin{thm}[Part of Theorem C.2.3 in \cite{kazhdan}] \label{thm:exist_f}
If $\psi \colon X \times X \to \R$ is a CND kernel on a topological space $X$, then there exists a real Hilbert space $H$ and a continuous mapping
\[
f \colon X \to H
\]
such that $\psi(x,y) = \|f(x) - f(y)\|_H^2$ for all $x, y \in X$. \qed
\end{thm}


From the above, it is straightforward to deduce:

\begin{cor}
If the geodesic Gaussian kernel is PD, then there exists a mapping $f \colon X \to H$ into some Hilbert space $H$ such that
\[
d(x,y) = \|f(x) - f(y)\|_H
\]
for each $x, y \in X$. Note that this mapping $f$ is \emph{not} necessarily related to the feature mapping $\phi \colon X \to V$ such that $k(x,y) = \langle \phi(x), \phi(y) \rangle_V$.
\end{cor}

\subsection{Curvature} \label{sec:curvature}

While curvature is usually studied using differential geometry, we shall access curvature via a more general approach that applies to general geodesic metric spaces. This notion of curvature, originating with Alexandrov and Gromov, operates by comparing the metric space to spaces whose geometry we understand well, referred to as \emph{model spaces}. The model spaces $M_\kappa$ are \emph{spheres} (of positive curvature $\kappa > 0$), the \emph{Euclidean plane} (flat, curvature $\kappa = 0$) and \emph{hyperbolic space} (negative curvature $\kappa < 0$). Since metric spaces can be pathological, curvature is approached by bounding the curvature of the space at a given point from above or below. The bounds are attained by comparing geodesic triangles in the metric space with triangles in the model spaces, as expressed in the \emph{$CAT(\kappa)$ condition}:

\begin{defn}
Let $(X, d)$ be a geodesic metric space $X$. Let $abc$ be a geodesic triangle of perimeter $< 2 D_\kappa$, where $D_\kappa$ is the diameter of $M_\kappa$,  that is, $D_\kappa = \infty$ for $\kappa \le 0$, and $D_\kappa = \frac{\pi}{\sqrt{\kappa}}$ for $\kappa > 0$. There exists a triangle $\bar{a}\bar{b}\bar{c}$ in the model space $M_\kappa$ with vertices $\bar{a}$, $\bar{b}$ and $\bar{c}$ and with geodesic edges $\bar{\gamma}_{\bar{a}\bar{b}}$, $\bar{\gamma}_{\bar{b}\bar{c}}$ and $\bar{\gamma}_{\bar{a}\bar{c}}$, whose lengths are the same as the lengths of the edges $\gamma_{ab}$, $\gamma_{bc}$ and $\gamma_{ac}$ in $abc$. This is an $M_\kappa$-\emph{comparison triangle} for $abc$ (see Fig.~\ref{cat0}).

For any point $x$ sitting on the segment $\gamma_{bc}$, there is a corresponding point $\bar{x}$ on the segment $\bar{\gamma}_{\bar{b} \bar{c}}$ in the comparison triangle, such that $d_{M_\kappa}(\bar{x},\bar{b}) = d(x,b)$. If we have 
\begin{equation} \label{cat0ineq}
d(x, a) \le d_{M_\kappa}(\bar{x},\bar{a})
\end{equation}
for every such $x$, and similarly for any $x$ on $\gamma_{ab}$ or $\gamma_{ac}$, then the geodesic triangle $abc$ satisfies the \emph{$CAT(\kappa)$ condition}.

The metric space $X$ is a $CAT(\kappa)$ space if any geodesic triangle $abc$ in $X$ of perimeter $< 2 D_\kappa$ satisfies the $CAT(\kappa)$ condition given in eq.~\ref{cat0ineq}. Geometrically, this means that triangles in $X$ are \emph{thinner} than triangles in $M_\kappa$.

The metric space $X$ has \emph{curvature $\le \kappa$ in the sense of Alexandrov} if it is locally $CAT(\kappa)$.
\end{defn}

\begin{figure}
\begin{center}
\includegraphics[width=0.7\linewidth]{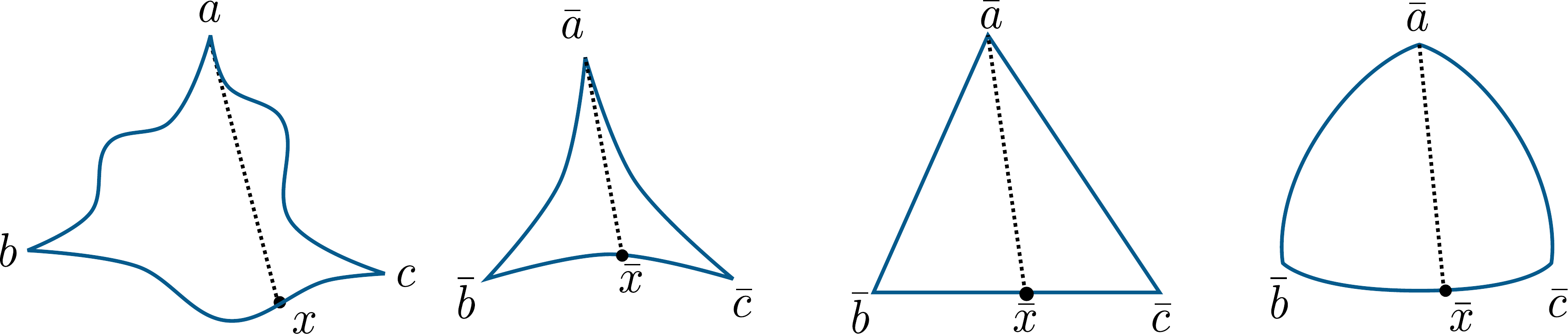}
\end{center}
\vspace{-3mm}
\caption{Left: A geodesic triangle, right: the corresponding comparison triangles in hyperbolic space $\mathbb{H}^2$, the plane $\R^2$ and the sphere $\mathbb{S}^2$, respectively.}
\vspace{-1mm}
\label{cat0}
\end{figure}

While curvature in the $CAT(\kappa)$ sense allows the study of curvature through the relatively simple means of geodesic distances alone, it is a weaker concept of curvature than the standard sectional curvature used in Riemannian geometry. Nevertheless, the two concepts are related, as captured by the following theorem due to Cartan and Alexandrov:

\begin{thm}[Theorem II.1A.6~\cite{bridsonhaef}] \label{curveq}
A smooth Riemannian manifold $M$ is of curvature $\le \kappa$ in the sense of Alexandrov if and only if the sectional curvature of $M$ is $\le \kappa$. \qed
\end{thm}

The proof of the main theorem will, moreover, rely on the following theorem characterizing manifolds of constant zero sectional curvature:

\begin{thm}[Part of Theorem 11.12~\cite{lee}] \label{sect_isomtoeucl}
Let $M$ be a complete, simply connected $m$-dimensional Riemannian manifold with constant sectional curvature $C=0$. Then $M$ is isometric to $\R^m$. \qed
\end{thm}

We are now ready to start proving our main theorems.


\subsection{Geodesic Gaussian kernels on metric spaces: Proof of Theorem~\ref{Xisflat}}

As in the statement of Theorem~\ref{Xisflat}, assume that the metric space $(X, d)$ is a geodesic space as defined in Sec.~\ref{sec:mainres}, and that $k(x,y) = e^{-\lambda d^2(x,y)}$ is a PD geodesic Gaussian kernel on $X$ for all $\lambda > 0$. 

An important consequence of Theorem~\ref{thm:exist_f} is that the map $f \colon X \to H$ must take geodesic segments to geodesic segments, which in $H$ are straight line segments. 

\begin{lem} \label{geo2geo}
If $\gamma \colon [0, L] \to X$ is a geodesic of length $L$ from $a = \gamma(0)$ to $b = \gamma(L)$ in $X$, then $f(\gamma([0, L]))$ is the straight line from $f(a)$ to $f(b)$ in $H$, and 
\begin{equation} \label{geo_img}
f\left( \gamma(t)\right) = f(a) + \frac{t}{L}\left(f(b)-f(a) \right)
\end{equation}
for all $t \in [0, L]$.
\end{lem}

\begin{proof}
Since $\gamma \colon [0, L] \to X$ is a geodesic, it contains every point $\gamma (t)$ for all $t \in [0, L]$, and since $\gamma$ is a geodesic of length $L$, we have $d\left(\gamma(0), \gamma (t) \right) = t$ for each $t \in [0,L]$, so
\[
\| f\left(\gamma(0) \right) - f\left( \gamma (t) \right) \| = d\left(\gamma(0), \gamma (t) \right) = t.
\]
This is only possible if $f \circ \gamma$ is the straight line from $f(a)$ to $f(b)$ in $H$. Equation~\eqref{geo_img} follows directly, as it is the geodesic parametrization of a straight line from $f(a)$ to $f(b)$.
\end{proof}

This enables us to prove Theorem~\ref{Xisflat}:

\begin{proof}[Proof of Theorem~\ref{Xisflat}]
Let $a, b, c \in X$ be three points in $X$ and form a geodesic triangle spanned by their joining geodesics $\gamma_{ab}, \gamma_{bc}$ and $\gamma_{ca}$. Then the points $f(a), f(b)$ and $f(c)$ in $H$ are connected by straight line geodesics $f \circ \gamma_{ab}, f \circ \gamma_{bc}$ and $f \circ \gamma_{ca}$ by Lemma~\ref{geo2geo}. These points and geodesics in $H$ span a $2$-dimensional linear subspace of $H$ in which they form a Euclidean comparison triangle.

Without loss of generality, pick any two points $x$ and $y$ on the geodesic triangle and measure the distance $d(x, y)$. The corresponding distance in the comparison triangle is $\|f(x) - f(y)\|$, and by the definition of $f$ we know that $d(x,y) = \|f(x) - f(y)\|$, so the geodesic triangle is isometrically embedded into the comparison triangle. Hence, $X$ is flat in the sense of Alexandrov.
\end{proof}

\begin{cor} \label{contractible}
The metric space $X$ is contractible, and hence simply connected.
\end{cor}

\begin{proof}
By Theorem~\ref{Xisflat}, $X$ must necessarily be a $CAT(0)$, and contractible by~\cite[Corollary II.1.5]{bridsonhaef}.
\end{proof}


\subsection{Geodesic Gaussian kernels on Riemannian manifolds: Proof of Theorem~\ref{isomeucl}}

As in the statement of Theorem~\ref{isomeucl}, assume that $M$ is a complete, smooth Riemannian manifold with associated geodesic distance metric $d$, and that $k(x,y) = e^{-\lambda d^2(x,y)}$ is a geodesic Gaussian kernel which is PD for all $\lambda > 0$. We prove that then, the Riemannian manifold $M$ is isometric to a Euclidean space.

\begin{proof}[Proof of Theorem~\ref{isomeucl}]
We start out by showing that the sectional curvature of $M$ is $0$ everywhere.

By Theorem~\ref{Xisflat}, $M$ is a $CAT(0)$ space, so in particular it has curvature $\le 0$ in the sense of Alexandrov. Therefore, by Theorem~\ref{curveq}, the sectional curvature of $M$ is $\le 0$. 

To prove the claim, we need to show that $M$ does not have any points with negative sectional curvature. To this end, assume that there is some point $p \in M$ such that the sectional curvature of $M$ at $p$ is $\kappa < 0$. Then, since sectional curvature on smooth Riemannian manifolds is continuous, there exists some neighborhood $U$ of $p$ and some $\kappa' < 0$ such that the sectional curvature in $U$ is $\le \kappa' < 0$. But then, by Theorem~\ref{curveq}, $U$ also has curvature $\le \kappa'$ in the sense of Alexandrov, which cannot hold due to Theorem~\ref{Xisflat}. It follows that the sectional curvature of $M$ at $p$ cannot be $\kappa < 0$; hence, the sectional curvature of $M$ must be everywhere $0$.

Since $M$ is simply connected by Corollary~\ref{contractible}, we apply Theorem~\ref{sect_isomtoeucl} to conclude that $M$ must be isometric to $\R^m$.
\end{proof}

\subsection{The case $q > 2$}

\begin{proof}[Proof of Theorem~\ref{nothigherthan2}]
This is a direct consequence of~\cite[Theorem~2.12]{istas}.
\end{proof}

\subsection{Geodesic Laplacian kernels:\\ Proof of Theorem~\ref{thm:PDlaplacian}} \label{sec:laplacian}

Another consequence of Sch\"{o}nberg's Theorem~\ref{thm:schonberg} is that the geodesic Laplacian kernel defined by~\eqref{eq:geod_kerneldef} with $q=1$ will be PD if and only if the distance $d$ is CND. This provides a PD kernel framework which can, for several popular Riemannian data manifolds, take advantage of the geodesic distance.

\begin{proof}[Proof of Theorem~\ref{thm:PDlaplacian}]
\begin{packed_itemize}
\item[i)] Let $(X, d)$ be a geodesic metric space. By Theorem~\ref{thm:schonberg}, $d$ is CND if and only if the Laplacian kernel $k(x,y) = e^{-\lambda d(x,y)}$ is PD for all $\lambda > 0$.
\item[ii)] By Theorem~\ref{thm:exist_f}, there exists a real Hilbert space $H$ and a continuous map $f \colon X \to H$ such that
\begin{equation} \label{dist_eq}
d(x,y) = \|f(x) - f(y)\|_H^2 \textrm{ for all } x,y \in X.
\end{equation}
That is, $d_{\sqrt{\ }}(x,y) = \|f(x) - f(y)\|_H$ for all $x, y \in X$. The map $f$ must be injective, because if $f(x) = f(y)$ for $x \neq y$ then by~\eqref{dist_eq}, $0 = \|f(x) - f(y)\|_H = d(x,y) > 0$, which is false. Therefore, $d_{\sqrt{\ }}$ coincides with the restriction to $f(X)$ of the metric on $H$ induced by $\| \cdot \|$. Since the restriction of a metric to a subset is a metric, $d_{\sqrt{\ }}$ is a metric, and by definition, $f$ is an isometric embedding of $(X, d_{\sqrt{\ }})$ into $H$.
\item[iii)] Since $f$ is an isometric embedding as metric spaces, $d_{\sqrt{\ }}$ must correspond to the chordal metric in $H$. 

Assume that $d_{\sqrt{\ }}$ is a geodesic metric on $X$, then by Lemma~\ref{geo2geo}, $f$ maps geodesics in $(X, d_{\sqrt{\ }})$ to straight line segments in $H$. Focusing on a single geodesic segment $\gamma \colon [0, L] \to X$, we obtain
\[
d_{\sqrt{\ }}\left(\gamma(t), \gamma(t') \right) = \|f \circ \gamma(t) - f \circ \gamma(t') \| = |t-t'|
\]
for all $t, t' \in [0, L]$. Since $d = d_{\sqrt{\ }}^2$ is a metric by assumption, the square $d_\gamma(t, t') = |t-t'|^2 = d\left(\gamma(t), \gamma(t') \right)$ is a metric on $[0, L]$. But this is not true, as the triangle inequality fails to hold.

Therefore, $d_{\sqrt{\ }}$ cannot be a geodesic metric on $X$.
\end{packed_itemize}
\vspace{-0.45cm}
\end{proof}

As noted in Table~\ref{tab:PD_or_not} below, for a number of popular Riemannian manifolds, the geodesic distance metric is CND, meaning that geodesic Laplacian kernels are PD. 

\begin{rem}
For a CND distance metric $d \colon \mathcal{X} \times \mathcal{X} \to \R$, a second way of constructing a PD kernel $k \colon \mathcal{X} \times \mathcal{X} \to \R$ is through the formula $k(x,x') = d(x,x') - d(x,x_0) -d(x_0, x')$~\cite{scholkopf,berg}, where $x_0 \in \mathcal{X}$ is any point. For other kernels based on distances, such as the rational-quadratic kernel, little is known.
\end{rem}

%
%

\section{Implications for popular manifolds and related work} \label{sec:relatedwork}

\begin{table*}
\centering
\resizebox{\textwidth}{!}{
\begin{tabular}{c|c|ccccc}
Space & Distance metric & Geodesic & Euclidean? & CND &  PD Gaussian & PD Laplacian\\
 & & metric? & metric? & metric? & kernel? & kernel? \\
\hline
$\R^n$~\cite{schonberg,learningwithkernels} & Euclidean metric & \textcolor{green}{\checkmark} & \textcolor{green}{\checkmark} & \textcolor{green}{\checkmark} & \textcolor{green}{\checkmark} & \textcolor{green}{\checkmark}\\
$\R^n$, $n > 2$~\cite{istas} & $l_q$-norm $\| \cdot \|_q$, $q > 2$ & \textcolor{green}{\checkmark}  & \textcolor{red}{$\div$} & \textcolor{red}{$\div$} & \textcolor{red}{$\div$} & \textcolor{red}{$\div$}\\
Sphere $\mathbb{S}^n$~\cite{istas} & classical intrinsic & \textcolor{green}{\checkmark} & \textcolor{red}{$\div$} & \textcolor{green}{\checkmark} & \textcolor{red}{$\div$} & \textcolor{green}{\checkmark}\\
Real projective space $\mathbb{P}^n(\R)$~\cite{robertson} & classical intrinsic & \textcolor{green}{\checkmark} & \textcolor{red}{$\div$} & \textcolor{red}{$\div$} & \textcolor{red}{$\div$} & \textcolor{red}{$\div$}\\
Grassmannian & classical intrinsic & \textcolor{green}{\checkmark} & \textcolor{red}{$\div$} & \textcolor{red}{$\div$} & \textcolor{red}{$\div$} & \textcolor{red}{$\div$}\\ 
$Sym^+_d$ & Frobenius & \textcolor{green}{\checkmark} & \textcolor{green}{\checkmark} & \textcolor{green}{\checkmark} & \textcolor{green}{\checkmark} & \textcolor{green}{\checkmark}\\
$Sym^+_d$ & Log-Euclidean & \textcolor{green}{\checkmark} & \textcolor{green}{\checkmark} & \textcolor{green}{\checkmark} & \textcolor{green}{\checkmark} & \textcolor{green}{\checkmark}\\
$Sym^+_d$ & Affine invariant & \textcolor{green}{\checkmark} & \textcolor{red}{$\div$} & \textcolor{red}{$\div$} & \textcolor{red}{$\div$} & \textcolor{red}{$\div$}\\
$Sym^+_d$ & Fisher information metric & \textcolor{green}{\checkmark} & \textcolor{red}{$\div$} & \textcolor{red}{$\div$} & \textcolor{red}{$\div$} & \textcolor{red}{$\div$}\\
Hyperbolic space $\mathbb{H}^n$~\cite{istas} & classical intrinsic & \textcolor{green}{\checkmark} & \textcolor{red}{$\div$} & \textcolor{green}{\checkmark} & \textcolor{red}{$\div$} & \textcolor{green}{\checkmark}\\
$1$-dimensional normal distributions & Fisher information metric & \textcolor{green}{\checkmark} & \textcolor{red}{$\div$} & \textcolor{green}{\checkmark} & \textcolor{red}{$\div$} & \textcolor{green}{\checkmark}\\
Metric trees~\cite{valette_trees}, \cite[Thm 2.15]{istas} & tree metric & \textcolor{green}{\checkmark} & \textcolor{red}{$\div$} & \textcolor{green}{\checkmark} & \textcolor{red}{$\div$} & \textcolor{green}{\checkmark}\\
Geometric graphs (\eg $k$NN) & shortest path distance & \textcolor{green}{\checkmark} & \textcolor{red}{$\div$} & \textcolor{red}{$\div$} & \textcolor{red}{$\div$} & \textcolor{red}{$\div$}\\
Strings~\cite{cortes_rational} & string edit distance & \textcolor{green}{\checkmark} & \textcolor{red}{$\div$} & \textcolor{red}{$\div$} & \textcolor{red}{$\div$} & \textcolor{red}{$\div$}\\
Trees, graphs & tree/graph edit distance & \textcolor{green}{\checkmark} & \textcolor{red}{$\div$} & \textcolor{red}{$\div$} &  \textcolor{red}{$\div$} & \textcolor{red}{$\div$}\\
\end{tabular}}
\caption{For a set of popular metric and manifold data spaces and metrics, we record whether the metric is a geodesic metric, whether it is a Euclidean metric, whether it is a CND metric, and whether its corresponding Gaussian and Laplacian kernels are PD.}
\label{tab:PD_or_not}
\end{table*}

Many popular data spaces appearing in computer vision are \emph{not} flat, meaning that their geodesic distances are not CND and their geodesic Gaussian kernels will not be PD. Table~\ref{tab:PD_or_not} lists known results on CND status of some popular data spaces. In particular, the classical intrinsic metrics on $\R^n$, $\mathbb{H}^n$ and $\mathbb{S}^n$ are all CND\footnote{As a curious side note, this implies that $\sqrt{\|x-y\|}$ is a metric on $\R^n$.}. As the Fisher information metric on $1$-dimensional normal distributions defines the hyperbolic geometry $\mathbb{H}^2$~\cite{amari:book:2000}, it will give a CND geodesic metric. For projective space, on the other hand,~\cite{robertson} provides an example showing that the classical intrinsic metric is \emph{not} CND. As Grassmannians are generalizations of projective spaces, their geodesic metrics are therefore also not generally CND.

Symmetric, positive definite $(d \times d)$ matrices form another important data manifold, denoted $Sym^+_d$. While the popular Frobenius and Log-Euclidean~\cite{arsigny2005fast} metrics on $Sym^+_d$ are actually Euclidean, little is known theoretically about whether the geodesic distance metrics of non-Euclidean Riemannian metrics on $Sym^+_d$ are CND. In Sec.~\ref{sec:experiments} we show empirically that neither the affine-invariant metric~\cite{Pennec:IJCV:2006} nor the Fisher information metric on the corresponding fixed-mean multivariate normal distributions~\cite{amari:book:2000,atkinson1981rao} induce a CND geodesic metric. Note how the qualitatively similar affine-invariant and Log-Euclidean metrics differ in whether they generate PD exponential kernels.

Non-manifold data spaces are also popular, \eg the \emph{edit distance} on strings was shown not to be CND by Cortes \etal~\cite{cortes_rational}. As tree- and graph edit distances generalize string edit distance, the same holds for these. The metric along a metric tree, on the other hand, is CND. In Sec.~\ref{sec:experiments} we show empirically that this does not generalize to the shortest path metric on a geometric graph, such as the $k$NN  or $\epsilon$-neighborhood graphs often used in manifold learning~\cite{tenenbaum_global_2000,lleRoweis,belkin,alamgir}.

\subsection{Relation to previous work}

Several PD kernels on manifolds have appeared in the literature, some of them even Gaussian kernels based on distance metrics on manifolds such as spheres or Grassmannian manifolds, which we generally consider as curved manifolds. The reader might wonder how this is possible given the above presented results. The explanation is that the distances used in these kernels are \emph{not} geodesic distances and, in many cases, have little or nothing to do with the Riemannian structure of the manifold. We discuss a few examples.

\begin{ex}
In~\cite{jayasumana_cvpr2013}, a PD kernel is defined on $Sym^+_d$ by using a geodesic Gaussian kernel with the log-Euclidean metric \cite{Pennec:IJCV:2006}. The log-Euclidean metric is defined by pulling the (Euclidean) Frobenius metric on $Sym_d$ back to $Sym^+_d$ via the diffeomorphic matrix logarithm. Equivalently, data in $Sym^+_d$ is mapped into the Euclidean $Sym_d$ via the diffeomorphic log map, and data is analyzed there. The geodesic Gaussian kernel is PD because the Riemannian manifold is actually a Euclidean space. In such cases, the Riemannian framework only adds an unnecessary layer of complexity.
\end{ex}

\begin{ex}
In~\cite{jayasumana_cvpr2014}, radial kernels are defined on spheres $\mathbb{S}^n$ by restricting kernels on $\R^{n+1}$ to $\mathbb{S}^n$, giving radial kernels with respect to the chordal metric on $\mathbb{S}^n$. Due to the symmetry of $\mathbb{S}^n$, any kernel which is radial with respect to the chordal metric, will also be radial with respect to the geodesic metric on $\mathbb{S}^n$. This result is next used to define PD radial kernels on the Grassmannian manifold $\mathcal{G}^r_n$ and on the Kendall shape space $\mathcal{SP}^n$. However, these kernels are not radial with respect to the usual Riemannian metrics on these spaces, but with respect to the projection distance and the full Procrustes distance, respectively, both of which are \emph{not} geodesic distances with respect to any Riemannian metric on $\mathcal{G}^r_n$ and $\mathcal{SP}^n$, respectively\footnote{Assume that either of these metrics were a Riemannian geodesic distance metric. The family of PD radial kernels defined in~\cite{jayasumana_cvpr2014} on both $\mathcal{G}^r_n$ and $\mathcal{SP}^n$ include Gaussian kernels with the projection distance and the full Procrustes distance, respectively. By our previous results, if these were geodesic distances with respect to some Riemannian metric, this Riemannian metric would define a Euclidean structure on $\mathcal{G}^r_n$ and $\mathcal{SP}^n$, respectively. This is impossible, since these manifolds are both compact.}. These kernels, thus, have little to do with the Riemannian geometry of $\mathcal{G}^r_n$ and $\mathcal{SP}^n$.
\end{ex}

\begin{ex}
In~\cite{courty} it is noted that since the feature map $\phi$ corresponding to a Euclidean Gaussian kernel maps data onto a hypersphere $\mathbb{S}$ in the reproducing kernel Hilbert space $V$~\cite{learningwithkernels}, it might improve classification to consider the geodesic distance on $\mathbb{S}$ rather than the chordal distance from $V$. This is, however, done by projecting each $\phi(x) \in V$ onto the tangent space $T_{\phi(\tilde{x})} \mathbb{S}$ at a fixed base point $\phi(\tilde{x})$, where the linear kernel in $V$ is employed. This explains why the resulting kernel $k_{\tilde{x}}$ is PD: the kernel linearizes the sphere and, thereby, discards the spherical geometry.
\end{ex}

\begin{ex} 
In~\cite{honeine2010angular} and~\cite{jayasumana2013dicta}, geodesic Laplacian kernels are defined on spheres; as shown above, these are PD.
\end{ex}

\begin{ex}
In~\cite{jaakkola}, a kernel is defined on a general sample space $\mathcal{X}$ by selecting a generating probability distribution $P_\theta$ on $\mathcal{X}$ and defining a \emph{Fisher kernel} on $\mathcal{X}$. Denote by $M_\Theta$ the Riemannian manifold defined by a parametrized family of probability distributions $P_\theta$, $\theta \in \Theta$, on $\mathcal{X}$ endowed with the Fisher information metric. The kernel $k \colon \mathcal{X} \times \mathcal{X} \to \R$ is defined by mapping samples in $\mathcal{X}$ to the tangent space $T_\theta M_\Theta$ and applying the Riemannian metric at $P_\theta \in M_\Theta$. This is PD because the kernel is an inner product on data mapped into a Euclidean tangent space. Again, the statistical manifold is linearized and the resulting kernel does not fully respect its geometry.
\end{ex}

In several of these examples the data space is linearized by mapping to a tangent space or into a linear ambient space, which always gives a PD kernel. It should, however, be stressed that the resulting kernels neither respect the distances nor the constraints encoded in the original Riemannian structure. Thus, the linearization will inevitably remove the information that the kernel was aiming to encode.

In general, whenever a data space is embedded into a Euclidean/Hilbert space and the chordal metric is used in~\eqref{eq:geod_kerneldef}, this will give a PD kernel. In this way, by the Whitney embedding theorem~\cite{lee}, universal kernels can be defined on any manifold. These kernels will, however, disregard any constraints encoded by the geodesic distance. 

It is tempting to refer to the Nash theorem~\cite{nash}, which states that any Riemannian manifold can be isometrically embedded into a Euclidean space. Here, however, ''isometric embedding'' refers to a \emph{Riemannian} isometry, which preserves the \emph{Riemannian metric} (the smoothly changing inner product) --- not to be confused with a distance metric! Therefore, in a Riemannian isometric embedding $f \colon \mathcal{X} \to \R^n$ we typically have $d(x,y) \neq \|f(x) - f(y)\|$. A kernel based on chordal distances in a Nash embedding will, thus, not generally be related to the geodesic distance. 

Note, moreover, that the Nash theorem does not guarantee a unique embedding; in fact there are viable embeddings generating a wide range of distance metrics inherited from the ambient Euclidean space. Therefore, an exponential kernel based on the chordal metric will typically have little to do with the intrinsic Riemannian structure of the manifold.

There exist PD kernels that take full advantage of Riemannian geometry without relying on geodesic distances:

\begin{ex}
Gong \etal \cite{gong:cvpr:2012}
design a PD kernel for \emph{domain adaptation} using the geometry of the Grassmann
manifold: Let $S_1$ and $S_2$ be two low-dimensional subspaces of $\mathbb{R}^n$
estimated with PCA on two related data sets. This gives two points $x_1, x_2$ on the Grassmann
manifold. A test point can be projected into all possible subspaces along
the Grassmann geodesic connecting $x_1$ and $x_2$, giving an infinite dimensional
feature vector in a Hilbert space. Gong \etal \cite{gong:cvpr:2012} show how
to compute inner products in this Hilbert space in closed-form, thereby providing
a PD kernel which takes geometry into account without relying on geodesic distances. 
\end{ex}

\section{Experiments} \label{sec:experiments}
We now validate our theoretical results empirically. First, we generate $500$
randomly drawn symmetric PD matrices of size $3 \times 3$. We
compute the Gram matrix of both the Gaussian and Laplacian kernels under both
the affine-invariant metric~\cite{Pennec:IJCV:2006} and the Fisher information metric on the corresponding fixed-mean multivariate normal distributions
\cite{amari:book:2000,atkinson1981rao}. Fig.~\ref{fig:results}a
shows the eigenspectrum of the four different Gram matrices. All four kernels
have negative eigenvalues, which imply that none of them are positive definite. This empirically proves that neither the affine-invariant metric nor the Fisher information metric induce CND geodesic distance metrics in general, although we know this to hold for the Fisher information metric on $Sym^+_1 = \R_+$.


\begin{figure*}
  \centering
  \begin{tabular}{ccc}
    \footnotesize{(a) Symmetric positive definite matrices} &
    \footnotesize{(b) Unit sphere} &
    \footnotesize{(c) One-dimensional subspaces} \\
    \includegraphics[width=0.32\textwidth]{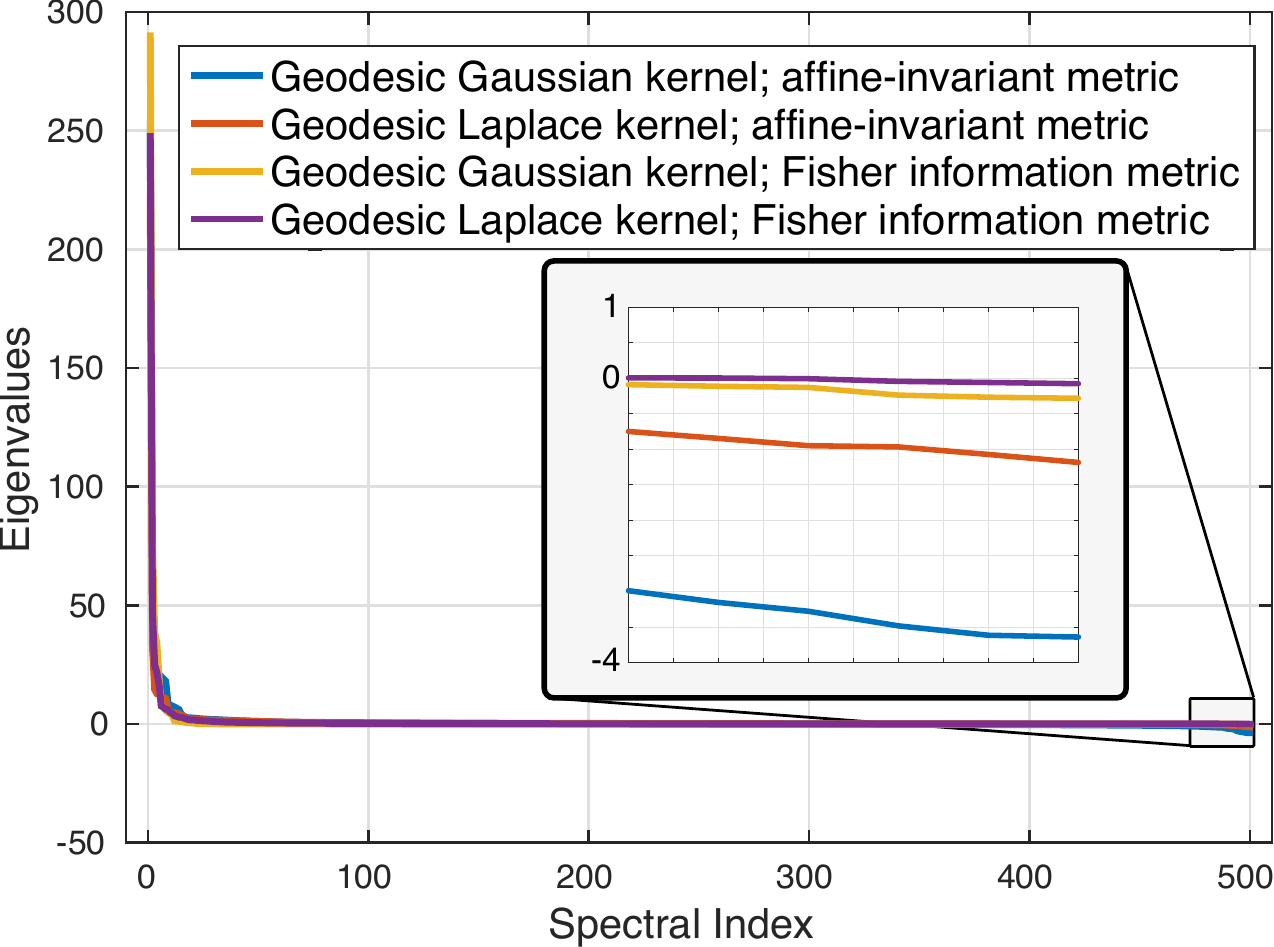} &
    \includegraphics[width=0.32\textwidth]{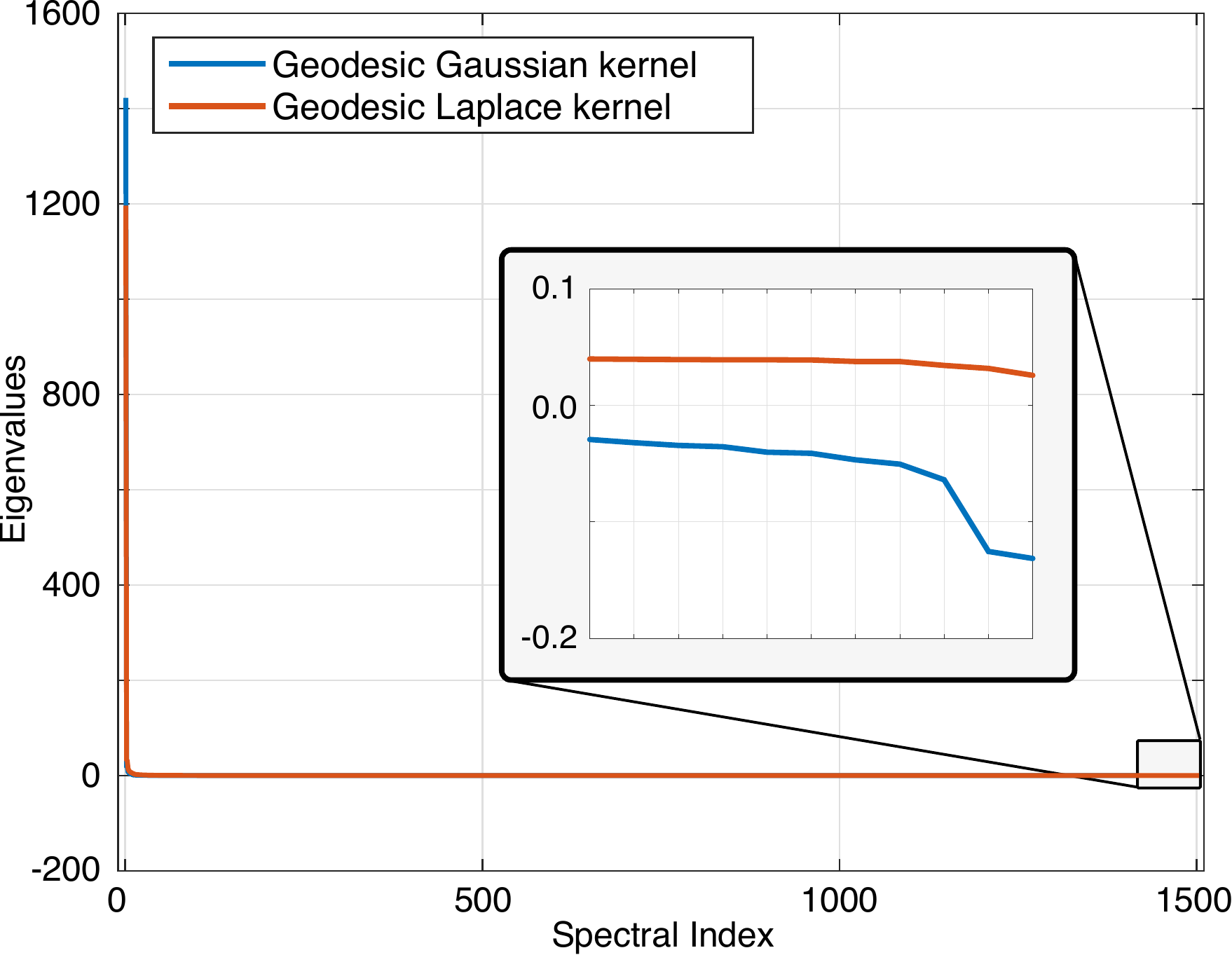} &
    \includegraphics[width=0.32\textwidth]{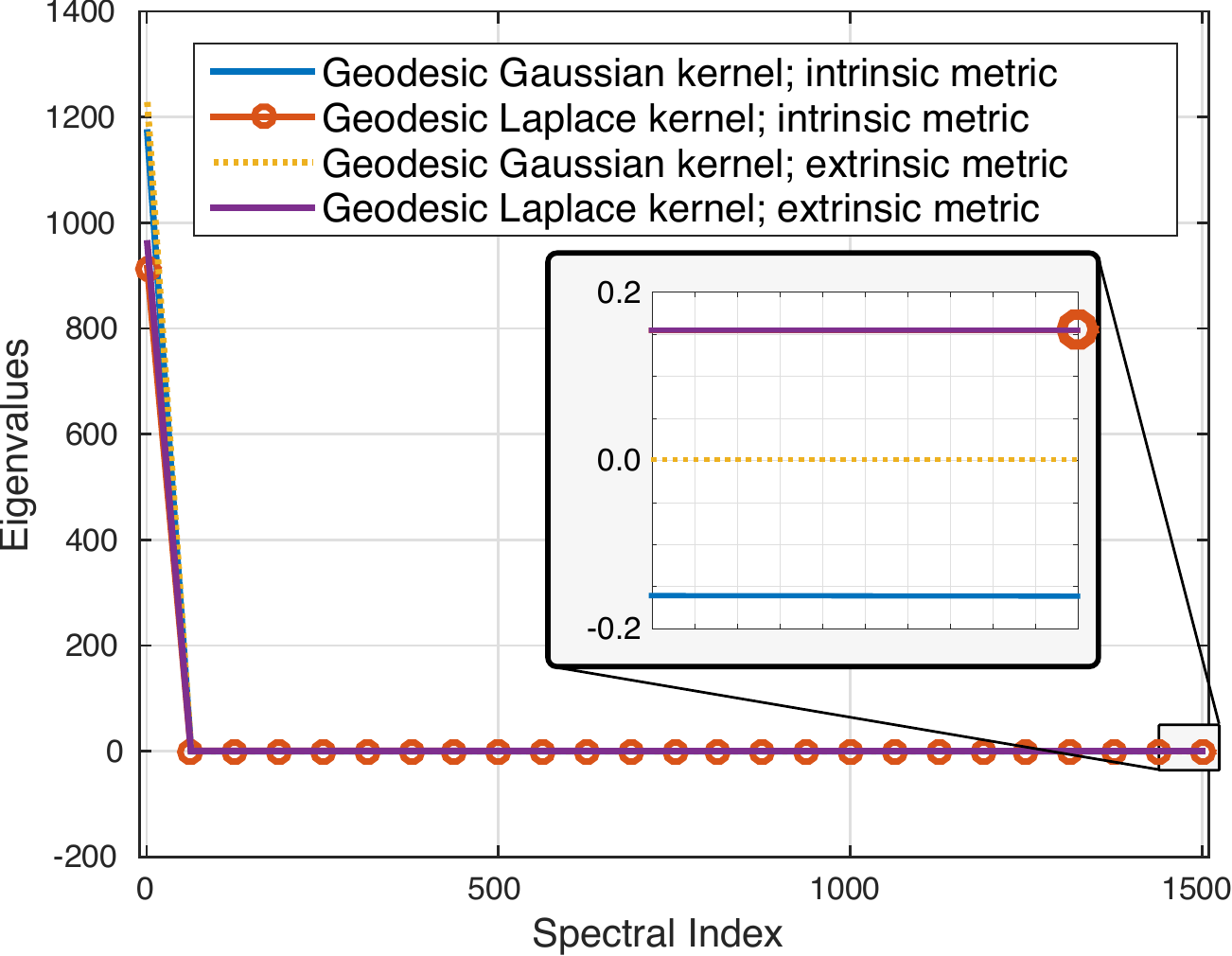} \\
    \footnotesize{(d) 15-dimensional subspaces of $\mathbb{R}^{100}$} &
    \footnotesize{(e) Nearest neighbor graph distances} &
    \footnotesize{(f) Example data} \\
    \includegraphics[width=0.32\textwidth]{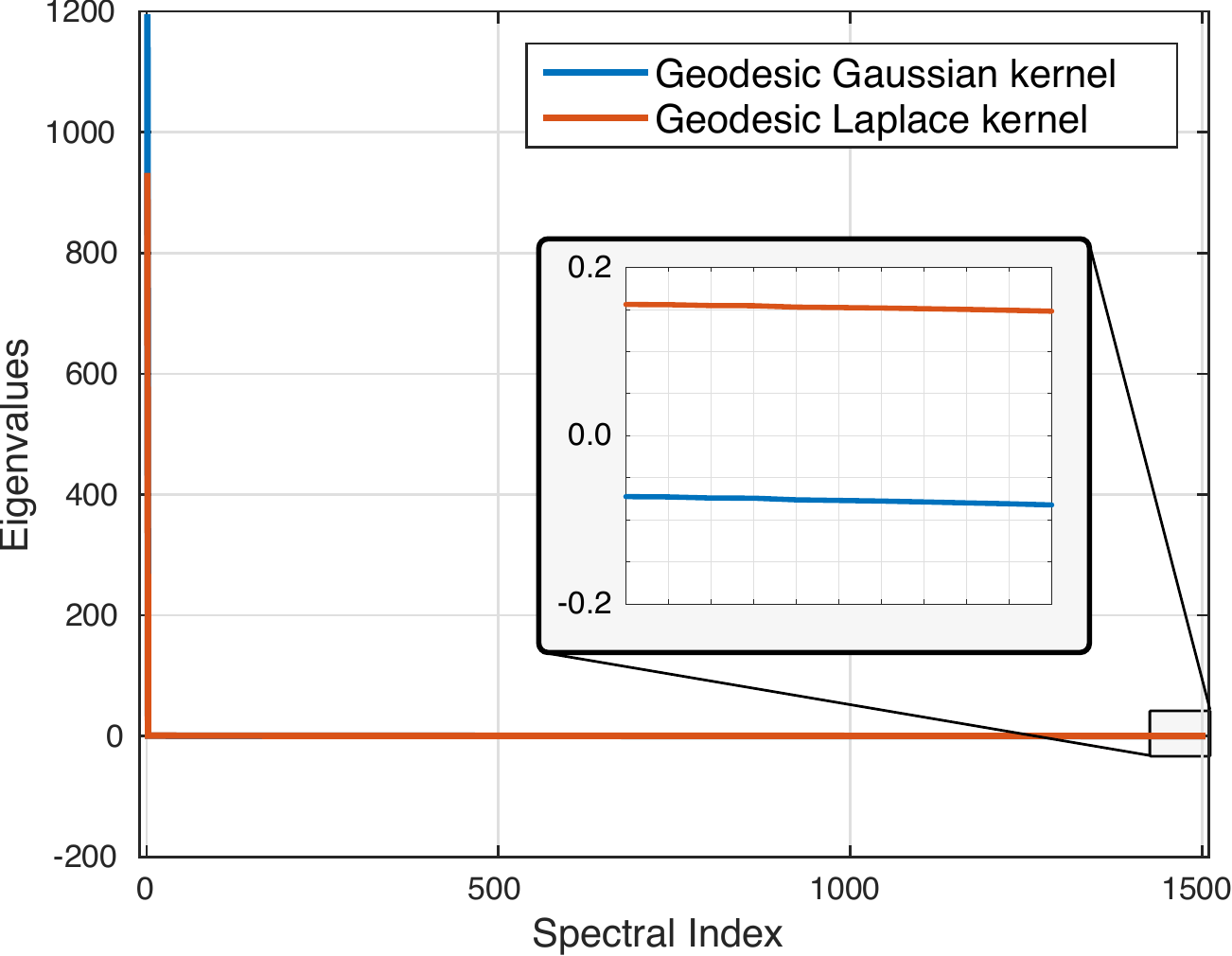} &
    \includegraphics[width=0.32\textwidth]{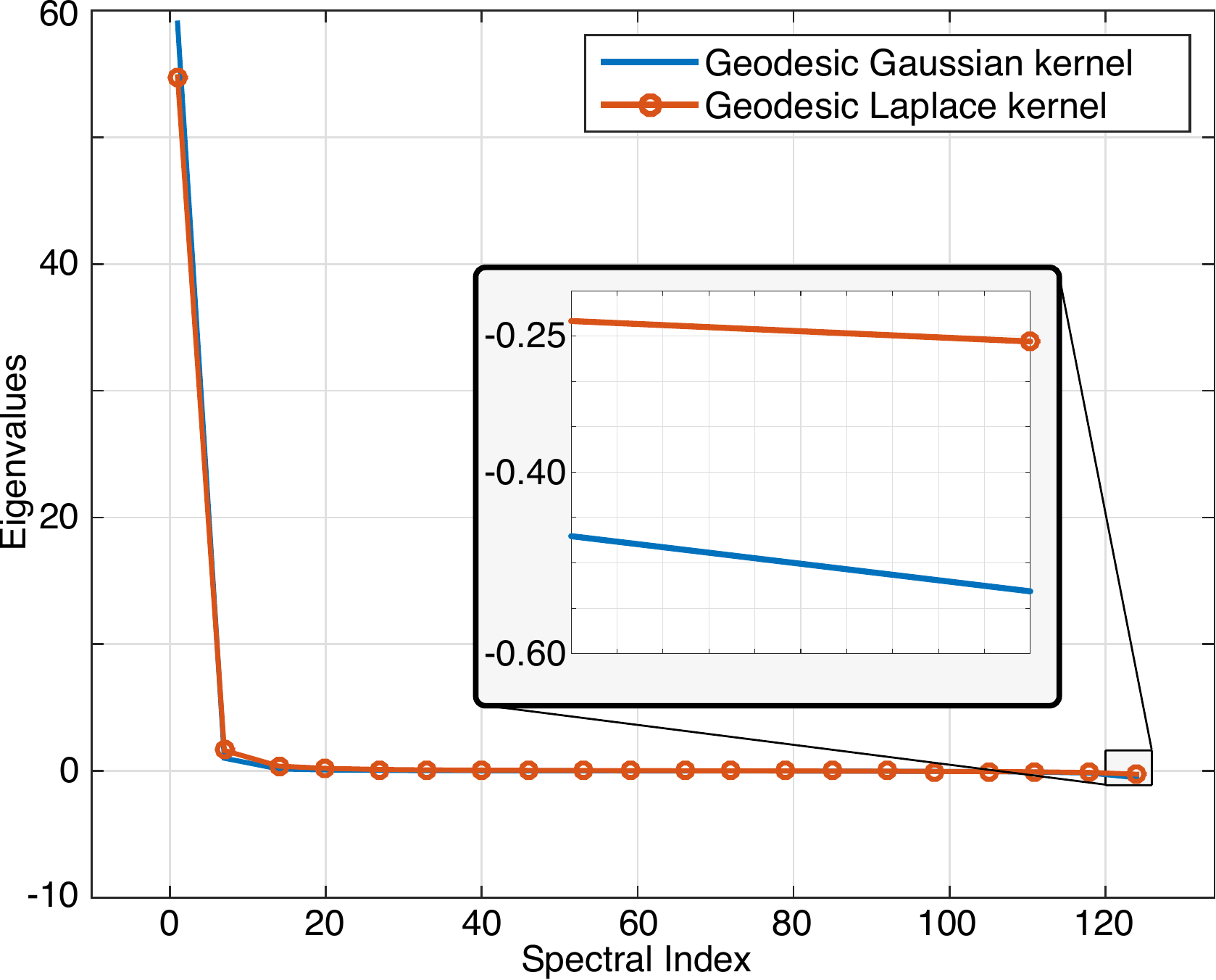} &
    \includegraphics[width=0.32\textwidth]{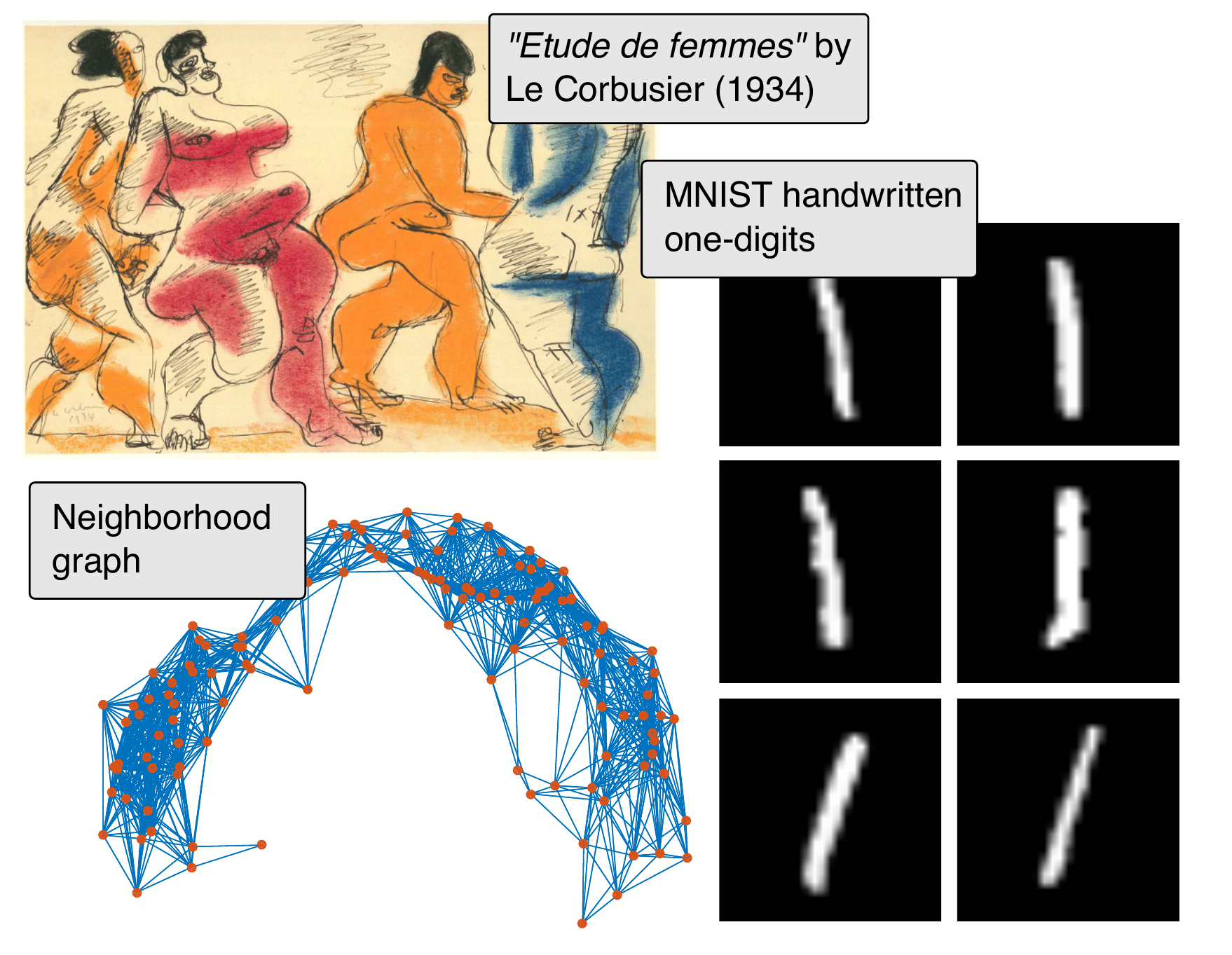}
  \end{tabular}
  \caption{(a)--(e): Eigenspectra of the Gram matrices for different geodesic exponential kernels on different manifolds. (f) Data used in panels b and e.}
  \label{fig:results}
\end{figure*}


Next, we consider kernels on the unit sphere. We generate data from salient points
in the 1934 painting \emph{Etude de femmes} by \emph{Le Corbusier}. At each salient
point a HOG \cite{dalal2005histograms} descriptor is computed; as these descriptors are normalized
they are points on the unit sphere. Fig.~\ref{fig:results}b shows the eigenspectrum
of the Gram matrix of the geodesic Gaussian and Laplacian kernels. While the geodesic Gaussian
kernel has negative eigenvalues, the geodesic Laplacian does not. This verifies our theoretical results from Sec.~\ref{sec:relatedwork} and Table~\ref{tab:PD_or_not}.


We also consider data on the Grassmann manifold. First, we consider one-dimensional
subspaces as spanned by samples from a $50$-dimensional isotropic normal distribution.
We again consider both the Gaussian and the Laplacian kernel; here both under the
usual intrinsic metric, but also under the extrinsic metric \cite{hauberg:cvpr:2014}.
Fig.~\ref{fig:results}c shows the eigenspectra of the different Gram matrices.
Only the Gaussian kernel under the intrinsic metric appears to have negative
eigenvalues, while the remaining have strictly positive eigenvalues.


Next, we consider $15$-dimensional subspaces of $\mathbb{R}^{100}$ drawn from a 
uniform distribution on the corresponding Grassmannian. We only consider kernels under the intrinsic metric, and the eigenspectra are shown in Fig.~\ref{fig:results}d. The Gaussian kernel has negative eigenvalues, while the
Laplacian kernel does not. Note that this does not prove that the Laplacian
kernel is PD on the Grassmannian; in fact, we know theoretically from~\cite{robertson} that it is generally \emph{not}.


Finally, we consider shortest-path distances on nearest neighbor graphs as
commonly used in manifold learning.
We take $124$ one-digits from the MNIST data set \cite{mnist}, project them
into their two leading principal components, form a $\epsilon$-neighborhood
graph, and compute shortest path distances. We then compute the eigenspectrum
of both the Gaussian and Laplacian kernel; Fig.~\ref{fig:results}e show these spectra. Both kernels have negative eigenvalues, which empirically show that
the shortest-path graph distance is \emph{not} CND.

\section{Discussion and outlook} \label{sec:discussion}

We have shown that exponential kernels based on geodesic distances in a metric space or Riemannian manifold will only be positive definite if the geodesic metric satisfies strong linearization properties:
\begin{packed_itemize}
  \item for Gaussian kernels, the metric space must be flat (or Euclidean).
  \item for Laplacian kernels, the metric must be conditionally negative definite. This implies that the square root metric can be embedded in a Hilbert space.
\end{packed_itemize}
With the exception of select metric spaces, these results show that geodesic exponential kernels are not well-suited for data analysis in curved spaces.

%

This does, however, \emph{not} imply that kernel methods can never be extended to metric spaces. Gong \etal \cite{gong:cvpr:2012} provide an elegant kernel based on the geometry of the Grassmann manifold, which is well-suited for domain adaptation. This kernel is not a geodesic exponential kernel, yet it strongly incorporates the geodesic structure of the Grassmannian. As an alternative, the Euclidean Gaussian kernel is a \emph{diffusion kernel}. Such kernels are positive definite on Riemannian manifolds \cite{lafferty:jmlr:2005}, and might provide a suitable kernel. However, these kernels generally do not have closed-form expressions, which may hinder their applicability.

%

Most existing machine learning tools assume a linear data space. Kernel methods only encode non-linearity via a non-linear transformation between a data space and a linear feature space. Our results illustrate that such methods are limited for analysis of data from non-linear spaces. Emerging generalizations of learning tools such as regression~\cite{hinkle,steinke,fletcher2} or transfer learning~\cite{xie:iccv:2013, oren} to nonlinear data spaces are encouraging. We believe that learning tools that operate directly in the non-linear data space, without a linearization step, is the way forward.

%
%
%
%
%
%
%
%
%


{\small
\bibliographystyle{ieee}
\bibliography{riemannian_kernels}
}

\end{document}